\documentclass{article}

\PassOptionsToPackage{numbers, compress}{natbib}

\usepackage[final]{neurips_2023}

\usepackage[utf8]{inputenc} 
\usepackage[T1]{fontenc}    
\usepackage{hyperref}       
\usepackage{url}            
\usepackage{booktabs}       
\usepackage{amsfonts}       
\usepackage{nicefrac}       
\usepackage{microtype}      
\usepackage{xcolor}         
\usepackage{esvect}

\usepackage{graphicx}
\usepackage{subcaption}
\usepackage{float}

\usepackage{hyperref}
\usepackage{amsmath}
\usepackage{amssymb}
\usepackage{mathtools}
\usepackage{amsthm}
\usepackage{multicol}
\usepackage{wrapfig}

\usepackage{algorithm}
\usepackage{algorithmic}

\usepackage{cleveref}
\theoremstyle{plain}
\newtheorem{theorem}{Theorem}[section]
\newtheorem{proposition}[theorem]{Proposition}
\newtheorem{lemma}[theorem]{Lemma}
\newtheorem{corollary}[theorem]{Corollary}
\theoremstyle{definition}

\newtheorem{assumption}[theorem]{Assumption}
\theoremstyle{remark}

\DeclareMathOperator*{\argmax}{arg\,max}
\DeclareMathOperator*{\argmin}{arg\,min}

\title{Generative Intrinsic Optimization: Intrinsic Control with Model Learning}

\author{%
  Jianfei Ma \\
  School of Mathematics and Statistics\\
  Northwestern Polytechnical University\\
  \texttt{matrixfeeney@gmail.com} \\
}

\begin{document}

\maketitle

\begin{abstract}
  Future sequence represents the outcome after executing the action into the environment (i.e. the trajectory onwards). When driven by the information-theoretic concept of mutual information, it seeks maximally informative consequences. Explicit outcomes may vary across state, return, or trajectory serving different purposes such as credit assignment or imitation learning. However, the inherent nature of incorporating intrinsic motivation with reward maximization is often neglected. In this work, we propose a policy iteration scheme that seamlessly incorporates the mutual information, ensuring convergence to the optimal policy. Concurrently, a variational approach is introduced, which jointly learns the necessary quantity for estimating the mutual information and the dynamics model, providing a general framework for incorporating different forms of outcomes of interest. While we mainly focus on theoretical analysis, our approach opens the possibilities of leveraging intrinsic control with model learning to enhance sample efficiency and incorporate uncertainty of the environment into decision-making.
\end{abstract}

\section{Introduction}
Deep reinforcement learning (RL) aims to improve an agent's policy with a task-specific reward, showing promise in solving complex tasks such as video games \cite{DBLP:journals/corr/MnihKSGAWR13} and robot locomotion \cite{DBLP:journals/corr/abs-1812-05905}. However, in many cases, obtaining a task-specific reward can be challenging, hindering the learning process. Intrinsic motivation, on the other hand, offers an alternative approach where the agent is driven by internal rewards to achieve goals or complete tasks. Its effectiveness has been shown in RL, including skill discovery \cite{DBLP:conf/iclr/GregorRW17}, curiosity-driven exploration \cite{DBLP:conf/nips/HouthooftCCDSTA16}, and representation learning \cite{DBLP:journals/corr/abs-1902-07685}. However, existing methods often treat intrinsic reward as an additional component to the task-specific reward, optimizing them using standard RL algorithms, without fully considering its unique nature in the agent's decision-making process. Furthermore, these methods often rely on specific variational approaches tailored to particular applications, lacking a unified perspective. In this work, we propose a novel approach that transforms the standard RL objective into a mutual information maximization framework, which employs a variational approach, enabling simultaneous approximation of the posterior and the transition model. This unified approach facilitates efficient intrinsic control combined with model learning.

In this paper, we present a comprehensive intrinsic control framework called Generative Intrinsic Optimization (GIO) that integrates a policy iteration scheme and a variational approach, enabling effective policy optimization by incorporating intrinsic motivation as a fundamental component of the agent's decision-making process. Our method is applicable to various future sequence forms, from one-step future sequences $\mathcal{F} = (s', r)$ to multi-step transitions, as well as compressed future sequences, offering potential synergies with existing approaches for further improvement. We provide a theoretical analysis of the convergence of our proposed scheme, ensuring monotonicity, and derive variational lower bounds for both one-step and multi-step scenarios. 
\begin{figure}[t]
\vskip 0.2in
\centering
\begin{minipage}{0.5\textwidth}
  \centering
  \includegraphics[width=\textwidth]{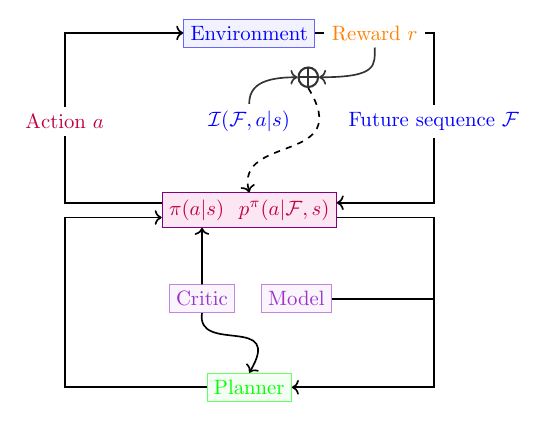}
  \subcaption{Algorithmic Architecture}
  \label{fig:4a}
\end{minipage}%
\begin{minipage}{0.5\textwidth}
  \centering
  \vspace{21.5pt}  
  \includegraphics[width=\textwidth]{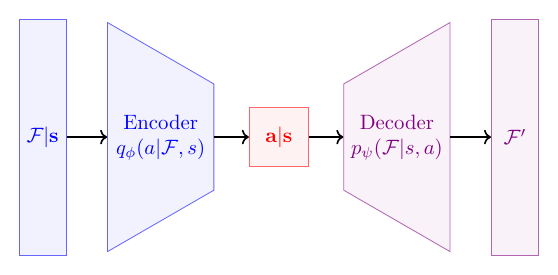}
  \vspace{21.5pt}
  \subcaption{Variational Diagram}
  \label{fig:4b}
\end{minipage}
\caption{\textbf{Left}: GIO combines both intrinsic and extrinsic rewards for policy learning and utilizes the learned model for planning purposes; \textbf{Right}: The variational model comprises an inference model and a generative model for posterior approximation and model learning.}
\label{fig:4}
\end{figure}

\section{Preliminaries}
\subsection{Notation}
Consider a regularized infinite-horizon discounted MDP, defined by a tuple $(\mathcal{S}, \mathcal{A}, P, r, \rho_{0}, \gamma, \Delta)$, where $\mathcal{S}$ is the state space, $\mathcal{A}$ is the action space, $P: \mathcal{S} \times \mathcal{A} \times \mathcal{S} \rightarrow \mathbb{R}$ is the transition probability distribution, $r: \mathcal{S} \times \mathcal{A} \rightarrow \mathbb{R}$ is the reward function, $\rho_{0}: \mathcal{S} \rightarrow \mathbb{R}$ is the distribution of the initial state $s_{0}$, $\gamma \in [0, 1)$ is the discount factor, and the additional term $\Delta$ represents other rewards such as intrinsic reward. We aim to maximize the objective function $\mathbb{E}_{\tau}\left[\sum\limits_{t=0}^{\infty}\gamma^{t}(r_{t} + \eta \Delta_{t})\right]$ with a temperature parameter $\eta$, where $\tau$ represents the trajectory generated by a stochastic policy $\pi : \mathcal{S} \times \mathcal{A} \rightarrow [0, 1]$. We denote the entropy of a distribution as $H(\cdot)$.
\paragraph{Information Seeking RL}
In the context of information-seeking RL, we introduce $\Delta = \mathcal{I}^{\pi}(\mathcal{F}, a | s)$ as the state-conditional mutual information between the current action and a future sequence $\mathcal{F}$ beyond the action execution. Our goal is to maximize the expected augmented reward by incorporating this mutual information term.
\begin{equation}
  \label{eq:mut}
  \eta (\pi) = \mathbb{E}_{\tau}\left[\sum\limits_{t=0}^{\infty}\gamma^{t}(r_{t} + \eta \mathcal{I}^{\pi}(\mathcal{F}_{t}, a_{t} | s_{t}))  \right]
\end{equation}
This formulation captures the uncertainty reduction between the current policy and the posterior, providing a flexible framework for various RL formulations. For instance, adopting an optimistic perspective, where the future sequence fully explains the executed action, the mutual information reduces to the entropy $\Delta = H(\pi)$, encouraging pure exploration \cite{DBLP:phd/us/Ziebart18}. In contrast, standard RL \cite{Sutton1998} takes a pessimistic stance, assuming the future sequence reveals no information about the executed action, that is, $\Delta = 0$.

In this paper, our focus will primarily be on the more general form of mutual information, allowing the incorporation of different choices of future sequences within a unified framework.

\section{Mutual Information}
Mutual information depicts mutual dependence between two random variables. Being an information-theoretic measure, it can be used to quantify the amount of information contained in the future $\mathcal{F}$ that explains the action $a$ given the current state $s$
\begin{equation}
  \label{eq:0}
\begin{aligned}
  \mathcal{I}^{\pi}(\mathcal{F}, a | s) = \mathbb{E}_{\pi(a | s) p(\mathcal{F} | s, a)}\left[\log \frac{p(\mathcal{F}, a | s)}{p(\mathcal{F} | s)\pi(a | s)}\right]
\end{aligned}
\end{equation}
where $\mathcal{F}$ can be any successor outcomes ahead of $(s, a)$, for instance one-step transition $\mathcal{F} = (s', r)$, or multi-steps subsequence $\mathcal{F} \subseteq (s'_{>}, r_{>})$. This quantity is compelling as it encourages the agent to seek maximally informative future outcomes and thereby reduce the uncertainty of the decisions.
\begin{equation}
  \label{eq:1}
  \mathcal{I}^{\pi}(\mathcal{F}, a | s) = H(\pi(a | s)) - \mathbb{E}_{p^{\pi}(\mathcal{F} | s)}\left[H(p^{\pi}(a | \mathcal{F}, s))\right]
\end{equation}
where $p^{\pi}(a | \mathcal{F}, s)$ is the posterior distribution corresponds to the prior $\pi$ after observing new outcomes. The mutual information quantifies the reduction in uncertainty between the prior and posterior. In what will follow, we present a policy iteration scheme that helps the agent pursue a policy that seeks maximum information about the future.
\section{Intrinsic Policy Iteration}
We start by deriving an intrinsic Bellman operator and proposing a policy iteration scheme. We then present a general convergence result for all valid future sequences $\mathcal{F}$.

It is useful to define the following operator
\begin{equation}
  \label{eq:3}
  \mathcal{T}^{\pi} Q(s_{t}, a_{t}) = r(s_{t}, a_{t}) + \gamma \mathbb{E}_{s_{t + 1}}[V(s_{t + 1})],
\end{equation}
where
\begin{equation}
  \label{eq:4}
  V(s_{t}) = \mathbb{E}_{a_{t} \sim \pi, \mathcal{F}_{t}}[Q(s_{t}, a_{t}) + \eta (\log{p^{\pi}(a_{t} | s_{t}, \mathcal{F}_{t})} - \log{\pi(a_{t} | s_{t})})]
\end{equation}
where $\eta$ is a hyperparameter that controls the relative strength of the augmentation against the reward.

It is not difficult to see that $\mathcal{T}^{\pi}$ is a contraction by modifying the reward as $r(s, a) + \gamma \mathbb{E}_{s'}\left[\mathcal{I}(a', \mathcal{F}' | s')\right]$. It indicates that if we repeatedly apply the intrinsic Bellman operator, we will get the intrinsic action-value function $Q^{\pi}$.
\begin{proposition}
  \label{prop:1}
  If $\mathcal{I}(\mathcal{F}, a | s)$ is bounded for any $s \in \mathcal{S}$, then $\lim_{k \rightarrow \infty} (\mathcal{T}^{\pi})^{k} Q = Q^{\pi}$ for any initial function $Q$, and specifically $Q^{\pi}$ is the unique solution of \eqref{eq:3}.
\end{proposition}

Although the intricate relationship between the posterior $p^{\pi}$ and the policy $\pi$ makes a direct improvement over $Q^{\pi}$ infeasible, it is possible to follow an alternating optimization procedure when both $Q^{\pi}$ and $p^{\pi}$ are fixed. In such cases, under certain conditions, this approach can still ensure optimality.

We can solve for the one-step optimal policy when the intrinsic action-value function is attained as follows
\begin{equation}
  \label{eq:5}
  \mathcal{G}(Q^{\pi}, p^{\pi}) = \frac{\exp{\frac{1}{\eta}\left(Q^{\pi} + \eta \mathbb{E}_{\mathcal{F}}\left[\log{p^{\pi}(a | s, \mathcal{F})}\right] \right)}}{Z^{\pi}(s)}
\end{equation}
where $Z^{\pi}(s)$ is a partition function dependent only on state $s$. 

Repeated application of the intrinsic Bellman operator and the softmax operator, we can produce a sequence of $Q^{\pi_{k}}, k = 0, 1, \cdots$ by starting from arbitrary policy $\pi_{0}$. Unsurprisingly, under some mild condition, for any future sequence $\mathcal{F}$ of interest, it is guaranteed to converge to the optimal policy $\pi^{\star} \triangleq \argmax_{\pi} V^{\pi}$ (where $V^{\pi}$ can be obtained by inserting $Q^{\pi}$ into Equation \eqref{eq:4}) and the optimal action-value function $Q^{\pi^{\star}}$.

\begin{assumption}
  \label{assum:1}
  The entropy $H(\pi^{\star})$ is bounded.
\end{assumption}
\begin{assumption}
  \label{assum:2}
  The initial policy $\pi_{0}$ is non-zero everywhere.
\end{assumption}
\begin{assumption}
  \label{assum:3}
  The limit of
\begin{equation}
  \sum\limits_{k=0}^{n}\mathbb{E}_{s', (a', \mathcal{F}') \sim (p^{\pi^{\star}}(\mathcal{F}', a' | s') - p^{\pi_{k + 1}}(\mathcal{F}', a' | s'))}\left[\eta (\log{p^{\pi^{\star}}(a' | \mathcal{F}', s')} - \log{\pi^{\star}(a' | s')}) + Q^{\pi_{k}}\right]
\end{equation}
  exists\footnote{We denote $\mathbb{E}_{p - q} = \mathbb{E}_{p} - \mathbb{E}_{q}$ for less verbatim repetition.} for any $(s, a) \in \mathcal{S} \times \mathcal{A}$.
\end{assumption}
\begin{theorem}
  \label{thm:converge}
  Under assumptions \ref{assum:1}--\ref{assum:3}, for any future sequence $\mathcal{F}$, it holds that
  \begin{equation}
    \label{eq:6}
    \lim_{k \rightarrow \infty} Q^{\pi_{k}} = Q^{\pi^{\star}}
  \end{equation}
\end{theorem}
However, without knowing the transition model and the posterior, it may be difficult to utilize this general convergence result. In the next section, we unify the model learning and posterior approximation into a single model, considering one-step transitions.
\section{Inference with Model Learning}

\subsection{Variational Inference}
Due to the intractability of the marginal distribution, obtaining the posterior can be challenging. Therefore, we employ variational inference \cite{DBLP:journals/corr/KingmaW13} using an inference model $q_{\phi}(a | s, s', r)$ to approximate the true posterior. The dynamic model is parameterized as $p_{\psi}(s', r | s, a)$. For a given policy $\pi$, we can derive a variational lower bound on the conditional marginal distribution of $\mathcal{F} = (s', r)$
\begin{equation}
  \label{eq:7}
\begin{aligned}
    \log p^{\pi}(s', r | s) & \geq \mathcal{L}(\phi, \psi; s, s', r)\\
  & = -D_{\text{KL}}(q_{\phi}(a | s, s', r) || \pi(a | s)) + \mathbb{E}_{q_{\phi}(a | s, s', r)}[\log p_{\psi}(s', r | s, a)]
\end{aligned}
\end{equation}
where the action space is naturally treated as a latent inference target, for which the policy contains the necessary prior knowledge. The recognition model encodes the sequential experiences to infer the true posterior and the generative model constructs environment dynamics. This allows efficient posterior approximation for any future sequence $\mathcal{F}$, capturing complex dynamics.

In practice, we make a common assumption of factorization for the transition model $p_{\psi}$
\begin{equation}
  \label{eq:8}
  p_{\psi}(s', r | s, a) = p_{\psi}(s' | s, a) p_{\psi}(r | s, a)
\end{equation}

\subsection{Policy Improvement}
After observing new outcomes emitted from the environment, the agent will update its belief over the current policy based on both extrinsic and intrinsic rewards. We project the policy onto the one-step optimal policy $\mathcal{G}(Q^{\pi}, p^{\pi})$ as shown in Equation \eqref{eq:5} for each state $s \in \mathcal{S}$
\begin{equation}
  \label{eq:9}
  \argmin_{\pi' \in \Pi}D_{\text{KL}}\left(\pi'(\cdot | s) \Bigl| \Bigr| \mathcal{G}(Q^{\pi}, p^{\pi}) \right)
\end{equation}
Once we arrive at our new policy, by reevaluating the corresponding posterior, we can guarantee a monotonic improvement.
\begin{theorem}
  \label{thm:monotonic}
  If $\tilde{\pi}$ minimizes the projection loss against $\pi$ for any $s \in \mathcal{S}$ and the corresponding posterior is attained. Then $Q^{\tilde{\pi}}(s_{t}, a_{t}) \geq Q^{\pi}(s_{t}, a_{t})$ for all $(s_{t}, a_{t}) \in \mathcal{S} \times \mathcal{A}$.
\end{theorem}
\subsection{Trajectory-Wise Lower Bound}
Alternatively, we can consider a lower bound upon a trajectory $\tau$ with a finite horizon $T$, enabling the capture of longer horizon information with a more accurate prediction of the transition and wider coverage of the trajectory surprise. Denote $\mathbf{x}_{1:T} = {(s_{t + 1}, r_{t} | s_{t})}_{t = 1}^{T}$ and $\mathbf{z}_{1:T} = {(a_{t} | s_{t})}_{t = 1}^{T}$, then we have
\begin{equation}
  \label{eq:10}
  \log p^{\pi} \left(\mathbf{x}_{1:T} \right) \geq \sum_{\tau=1}^{T} \mathbb{E}_{\mathbf{z}_{1:\tau}} \left[\log p_{\psi} \left(s'_{\tau}, r_{\tau} | s_{\tau}, a_{\tau}\right) - D_{KL}\left(q_\phi \left(a_{\tau} | s_{1:\tau + 1}, r_{1:\tau}, a_{1:\tau - 1}\right)\Vert \pi \left(a_{\tau} | s_{\tau}\right)  \right) \right]
\end{equation}
where we assume factorization of the recognition and generative model and also use the Markov property of the policy distribution.

If we further assume conditional independence of the past information i.e. $t < \tau$ for $q_{\phi}$, then it reduces to a compact formulation being a summation of a series of one-step lower bounds (Equation \eqref{eq:7}). It can be useful for inferring multi-step posteriors simultaneously and fitting a transition model with a longer temporal dependence. Although it is promising to combine with techniques such as RNN \cite{DBLP:journals/corr/abs-1808-03314} or transformer \cite{DBLP:journals/corr/VaswaniSPUJGKP17}, it is beyond our scope and can be a further enhancement of our method.

\section{Related Work}
\paragraph{Intrinsic Motivation}
Intrinsic motivation is the drive to seek out and engage in activities that promote learning, exploration, and curiosity-driven behavior \cite{DBLP:journals/finr/OudeyerK09}. Mutual information has proven effective in diverse domains, including curiosity-driven exploration \cite{DBLP:conf/nips/HouthooftCCDSTA16, DBLP:conf/icml/KimKJLS19}, options discovery \cite{DBLP:conf/iclr/GregorRW17, DBLP:conf/iclr/EysenbachGIL19}, and empowerment maximization \cite{DBLP:conf/nips/MohamedR15, DBLP:conf/isrr/KarlBSBSB19}. Our method differs from empowerment maximization as we do not learn an open-loop distribution to maximize channel capability \cite{DBLP:conf/cec/KlyubinPN05}. Additionally, there exist other techniques for motivating agents from different perspectives, such as model uncertainty \cite{DBLP:conf/icml/PathakG019, DBLP:conf/icml/ShyamJG19, DBLP:journals/corr/abs-1902-07685, DBLP:conf/nips/FuCL17}, count-based exploration \cite{DBLP:conf/icml/OstrovskiBOM17, DBLP:conf/nips/BellemareSOSSM16, DBLP:conf/nips/TangHFSCDSTA17}, and surprise or novelty \cite{DBLP:journals/corr/AchiamS17, DBLP:conf/agi/SunGS11, DBLP:conf/iclr/BurdaEPSDE19}. While entropy is commonly used in model-free algorithms, we treat it as an intrinsic motivation solely encouraging exploration, in contrast to our nature of uncertainty reduction from the environment. \cite{DBLP:conf/nips/LeibfriedPG19} unifies reward and empowerment maximization, but requires extensive application of the Blahut-Arimoto algorithm, whose complexity has posed a challenge to scale to the continuous domain \cite{DBLP:conf/nips/MohamedR15}. Another essential difference is that our method focuses on an efficient policy iteration approach analogous to \cite{DBLP:conf/icml/HaarnojaZAL18}, whereas \cite{DBLP:conf/nips/LeibfriedPG19} manipulates the optimality operator similar to \cite{DBLP:conf/icml/HaarnojaTAL17}, but with far less flexibility on $\mathcal{F}$. 

\paragraph{Incorporating the Future}
\cite{DBLP:conf/iclr/KeSTGBPB19} incorporates information from future observations and actions using a bidirectional recurrent network in an autoregressive manner. To address credit assignment, \cite{DBLP:conf/nips/HarutyunyanDMAP19} introduces the importance ratio between a state-conditional posterior and the policy, measuring the relevance of past decisions to the trajectory return or future state. Compressing the sequence into a compact representation reduces the challenges of long sequence modeling \cite{DBLP:conf/iclr/VenutoLPN22} \cite{DBLP:conf/iclr/KeSTGBPB19}. The same ideas can also be drawn from, particularly when $\mathcal{F}$ is too long to capture useful information, we can instead employ an additional variational model to construct a compact representation $\mathcal{F}_{z}$, which is then used for posterior approximation and model learning. RL Upside Down \cite{DBLP:journals/corr/abs-1912-02875} predicts actions using reward signals and states, resembling our posterior formulation. However, it diverges by redefining the policy, while our method solely informs it. Our approach provides a unified perspective, accommodating different forms of $\mathcal{F}$ within a policy iteration scheme, ensuring convergence.

\section{Conclusion}
In this paper, we introduce a novel learning framework that integrates intrinsic control with model learning. Our algorithm adapts to different types of future sequences, focusing on maximizing the informativeness of future outcomes given executed actions. It guarantees convergence and monotonicity. Our approach opens up possibilities for various algorithmic formulations, including trajectory-wise methods, imitation learning, and direct probabilistic control, with the full utilization of function approximations as future work.

\bibliographystyle{plain}
\small\bibliography{main.bib}

\newpage
\appendix
\onecolumn
\section{Pseudocode of GIO}

\begin{algorithm}[h]
\caption{Generative Intrinsic Optimization}
\label{alg:GIO}
\textbf{Input}: $\eta, \tau$\\
\textbf{Initial Parameter}: $\{q_{\phi}, p_{\psi}\}, Q_{\bar{w}_{i}}, Q_{w_{i}}, \pi_{\theta}$
\begin{algorithmic}
    \FOR{step $t \gets 0, 1, \dots, M - 1$}{
        \STATE Execute policy $\pi_{\theta}$ in the environment
        \STATE Store transition $(s_{t}, a_{t}, r_{t}, s_{t + 1})$ to the replay buffer $\mathcal{D}$
        \STATE Sample mini-batch of $n$ transitions $(s, a, r, s')$ from $\mathcal{D}$
        \STATE Train VAE based on the variational lower bound upon $(p_{\phi}, p_{\psi})$
        \STATE Evaluate $\Delta = (\log q_{\phi}(a' | s', s'', r' ) - \log \pi_{\theta}(a' | s'))$ with new sampled action $a' \sim \pi_{\theta}$ and experience $(s'', r') \sim p_{\psi}$
        \STATE Compute value target $y = r(s, a) + \gamma (\min_{i}Q_{\bar{w}_{i}}(s', a') + \eta \Delta)$
        \STATE Update critic with $\nabla_{w}\mathcal{J}(w)$ (Equation \eqref{eq:11})
        \STATE Update actor with $\nabla_{\theta}\mathcal{J}(\theta)$ (Equation \eqref{eq:12})
        \IF{it is time to plan}
        \STATE Update actor with simulated policy gradient
        \ENDIF
        \STATE Update target net $\bar{w}_{i} \gets (1 - \tau) \bar{w}_{i} + \tau w_{i}, i = 1, 2$
    }
    \ENDFOR
\end{algorithmic}
\end{algorithm}

We provide a potential learning procedure that utilizes the clipped double-Q technique \cite{DBLP:conf/icml/FujimotoHM18} and the reparameterized policy gradient. The action-value function and the policy are parameterized as $Q_{w}$ and $\pi_{\theta}$ respectively.

Denote $\Delta = (\log q_{\phi}(a' | s', s'', r' ) - \log \pi_{\theta}(a' | s'))$, the critic is updated by following fitted Q-iteration \cite{ernst2005tree} \cite{DBLP:conf/l4dc/FanWXY20}
\begin{equation}
  \label{eq:11}
  \mathcal{J}(w) = \mathbb{E}_{(s, a, s', r) \sim \mathcal{D}}\bigl[(Q_{w_{i}}(s, a) - (r + \gamma (\min_{i}Q_{\bar{w}_{i}}(s', a') + \eta \Delta)))^{2} \bigr], i = 1, 2
\end{equation}
where a new action $a' \sim \pi_{\theta}(\cdot | s')$ and experience $(s'', r') \sim p_{\psi}$ are sampled for evaluating the log ratio. The target value network $Q_{\bar{w}_{i}}$ is utilized to stablize the behavior of the neural networks, which is commonly used in off-policy algorithms \cite{DBLP:journals/corr/LillicrapHPHETS15} \cite{DBLP:conf/icml/HaarnojaZAL18} \cite{DBLP:conf/icml/FujimotoHM18}.

And the parameterized policy orients itself to the softmax policy w.r.t. the approximate action-value function and posterior.
\begin{equation}
  \label{eq:12}
  \mathcal{J}(\theta) = \mathbb{E}_{s \sim \mathcal{D}}\left[D_{\text{KL}}\left(\pi_{\theta}(\cdot | s) \Bigl| \Bigr| \mathcal{G}(Q_{w}, q_{\phi}) \right)\right]
\end{equation}
which can also utilize the reparametrization trick \cite{DBLP:journals/corr/KingmaW13}, resulting in a potential lower variance gradient estimator.

\section{Proof of Proposition \ref{prop:1}}
\begin{proof}
By plugging Equation \eqref{eq:4} into \eqref{eq:3}, we have
\begin{equation}
  \label{eq:14}
  \mathcal{T}^{\pi} Q(s_{t}, a_{t}) = r(s_{t}, a_{t}) + \gamma \mathbb{E}_{s_{t + 1}, a_{t + 1}, \mathcal{F}_{t + 1}}[Q(s_{t + 1}, a_{t + 1}) + \eta (\log{p^{\pi}(a_{t + 1} | s_{t + 1}, \mathcal{F}_{t + 1})} - \log{\pi(a_{t + 1} | s_{t + 1})})]
\end{equation}
If we merge the log ratio into reward such that $\tilde{r}_{t} = r_{t} + \gamma \mathbb{E}_{s_{t + 1}, a_{t + 1}, \mathcal{F}_{t + 1}}[\eta (\log{p^{\pi}(a_{t + 1} | s_{t + 1}, \mathcal{F}_{t + 1})} - \log{\pi(a_{t + 1} | s_{t + 1})})]$, we alternatively have
\begin{equation}
  \label{eq:15}
  \mathcal{T}^{\pi} Q(s_{t}, a_{t}) = \tilde{r}(s_{t}, a_{t}) + \gamma \mathbb{E}_{s_{t + 1}, a_{t + 1}}[Q(s_{t + 1}, a_{t + 1})]
\end{equation}

For any $Q_{1}, Q_{2}$ in the action-value space $\mathcal{Q}$
\begin{equation}
  \label{eq:16}
  \begin{aligned}
    \| \mathcal{T}^{\pi}Q_{1} - \mathcal{T}^{\pi}Q_{2} \|_{\infty} & = \sup_{s, a}\bigl|\tilde{r}(s, a) + \gamma \mathbb{E}_{s', a'}[Q_{1}(s', a')] - \tilde{r}(s, a) + \gamma \mathbb{E}_{s', a'}[Q_{2}(s', a')] \bigr| \\
                                                                                                             & = \gamma \sup_{s, a}\bigl|\mathbb{E}_{s', a'}[Q_{1}(s', a') - Q_{2}(s', a')] \bigr| \\
                                                                                                             & \leq \gamma \sup_{s, a}\bigl|\mathbb{E}_{s', a'}[\sup_{s', a'}| Q_{1}(s', a') - Q_{2}(s', a')| ] \bigr|\\
                                                                                                             & = \gamma \sup_{s, a} \sup_{s', a'} | Q_{1}(s', a') - Q_{2}(s', a') | \\
                                                                                                             & = \gamma \sup_{s', a'} | Q_{1}(s', a') - Q_{2}(s', a')| \\
                                                                                                             & = \gamma \| Q_{1} - Q_{2} \|_{\infty}
  \end{aligned}
\end{equation}
This implies $\mathcal{T}^{\pi}$ is a contraction mapping in the metric space $\mathcal{Q}$. From the Banach fixed-point theorem, we know that starting from any initial point $Q$, the sequence $Q_{k + 1} = \mathcal{T}^{\pi}Q_{k}$ converges to a unique fixed point $Q^{\star}$. Since $Q^{\pi}$ solves for $\mathcal{T}^{\pi}$ by definition, it implies $Q^{\pi} = Q^{\star}$.
\end{proof}
\section{Proof of Equation \eqref{eq:5}}
\label{sec:proof-equat-eqref-5}
\begin{proof}
Considering the one-step optimization problem\footnote{For simplicity, our derivation is based on the discrete case, however, the same procedure also applies for the continuous case likewise \cite{DBLP:conf/iclr/AbdolmalekiSTMH18}. Thus, the claims will not degenerate.} with the posterior $p^{\pi}$ and the value function $V^{\pi}$ being fixed
\begin{equation}
  \label{eq:17}
  \begin{aligned}
  V^{\tilde{\pi}, p_{\pi}}(s) & \triangleq \sup_{\pi}\mathbb{E}_{a \sim \pi, \mathcal{F}}\left[r(s, a) + \eta (\log{p^{\pi}(a | s, \mathcal{F})} - \log{\pi(a | s)}) + \gamma \mathbb{E}_{s'}[V^{\pi}(s')]\right] \\
  & = \sup_{\pi}\mathbb{E}_{a \sim \pi, \mathcal{F}}\left[\eta (\log{p^{\pi}(a | s, \mathcal{F})} - \log{\pi(a | s)}) + Q^{\pi}(s, a)\right]    
  \end{aligned}
\end{equation}
Define the Lagrangian function $\mathcal{L}(s; \lambda): \mathcal{S} \rightarrow \mathbb{R}$
\begin{equation}
  \label{eq:18}
  \mathcal{L}(s; \lambda) = \mathbb{E}_{a \sim \pi, \mathcal{F}}\bigr[\eta (\log{p^{\pi}(a | s, \mathcal{F})} - \log{\pi(a | s)}) + Q^{\pi}(s, a)\bigr] - \lambda (\sum\limits_{a \in \mathcal{A}}\pi(a | s) - 1)
\end{equation}
Solving for the first-order equation
\begin{equation}
  \label{eq:19}
  0 = \frac{\partial \mathcal{L}(s; \lambda)}{\partial \pi(a | s)} = Q^{\pi}(s, a) + \eta \mathbb{E}_{\mathcal{F}}\left[\log{p^{\pi}(a | s, \mathcal{F})}\right] - \eta \log{\pi(a | s)} - \eta - \lambda
\end{equation}
Rearranging
\begin{equation}
  \label{eq:20}
  \tilde{\pi} = \exp{(-\frac{\lambda}{\eta} - 1)} \exp{\frac{1}{\eta}(Q^{\pi} + \eta \mathbb{E}_{\mathcal{F}}\left[\log{p^{\pi}(a | s, \mathcal{F})}\right])}
\end{equation}
With the equality constraint
\begin{equation}
  \label{eq:21}
  \sum\limits_{a \in \mathcal{A}} \tilde{\pi}(a | s) = 1
\end{equation}
by applying log transformation on both sides, we can solve for the multiplier as
\begin{equation}
  \label{eq:22}
  \tilde{\lambda} = \eta \log{\sum\limits_{a \in \mathcal{A}} \exp{\frac{1}{\eta}\bigl(Q^{\pi} + \eta \mathbb{E}_{\mathcal{F}}\left[\log{p^{\pi}(a | s, \mathcal{F})}\right])}} - \eta
\end{equation}
inserting which into Equation \eqref{eq:20}, we get
\begin{equation}
  \label{eq:23}
  \tilde{\pi}(a | s) = \frac{\exp{\frac{1}{\eta}\bigl(Q^{\pi} + \eta \mathbb{E}_{\mathcal{F}}\left[\log{p^{\pi}(a | s, \mathcal{F})}\right])}}{\sum\limits_{a \in \mathcal{A}}\exp{\frac{1}{\eta}\bigl(Q^{\pi} + \eta \mathbb{E}_{\mathcal{F}}\left[\log{p^{\pi}(a | s, \mathcal{F})}\right])}}
\end{equation}
For the optimal policy, there must exist a multiplier that jointly satisfy KKT condition. Since $(\tilde{\pi}, \tilde{\lambda})$ uniquely satisfies the KKT condition as above, it implies $\tilde{\pi}$ is the optimal policy. Denote the denominator as $Z^{\pi}(s)$, it completes the proof.
\end{proof}
\section{Proof of Theorem \ref{thm:converge}}
\begin{lemma}
  \label{lm:kl}
  Let $p(x, y)$ be the joint distribution, and $p(x)$ and $p(y)$ be the marginal distribution correspondingly, then for any distribution $q(y)$, it holds that
  \begin{equation}
    \label{eq:24}
    D_{\text{KL}}(p(x, y) \| p(y)p(x)) \leq D_{\text{KL}}(p(x, y) \| q(y)p(x))
  \end{equation}
\end{lemma}
\begin{proof}
  Denote $\Gamma_{x, y}$ as either $\int_{x, y}$ for continuous case or $\sum\limits_{x, y}$ for discrete case, by non-negativity of KL divergence, it follows that
  \begin{equation}
    \label{eq:25}
    \begin{aligned}
    & D_{\text{KL}}(p(x, y) \| q(y)p(x)) - D_{\text{KL}}(p(x, y) \| p(y)p(x)) \\
    & = \Gamma_{x, y}p(x, y) \log {\frac{p(x, y)}{p(x) p(y)}} - \Gamma_{x, y}p(x, y) \log {\frac{p(x, y)}{p(x) q(y)}} \\
    & = \Gamma_{x, y}p(x, y) \log {\frac{p(y)}{q(y)}} \\
    & = \Gamma_{y}p(y) \log {\frac{p(y)}{q(y)}} \\
    & = D_{\text{KL}}(p(y) \| q(y)) \\
    & \geq 0      
    \end{aligned}
  \end{equation}
  which completes the proof.
\end{proof}
\begin{corollary}
  \label{cor:post}
  For any distribution $q(x | y)$, it holds that
  \begin{equation}
    \label{eqq:24}
    \mathbb{E}_{p(x, y)}\left[\frac{\log p(x | y)}{p(x)}\right] \geq \mathbb{E}_{p(x, y)}\left[\frac{\log q(x | y)}{p(x)}\right]
  \end{equation}  
\end{corollary}
\begin{proof}
  The proof is similar to that of the previous lemma, by non-negativity of KL divergence, it follows that
  \begin{equation}
    \label{eqq:25}
    \begin{aligned}
    & \mathbb{E}_{p(x, y)}\left[\frac{\log p(x | y)}{p(x)}\right] - \mathbb{E}_{p(x, y)}\left[\frac{\log q(x | y)}{p(x)}\right] \\
    & = \Gamma_{x, y}p(x, y) \log {\frac{p(x | y)}{q(x | y)}} \\
    & = \Gamma_{x, y} p(y) p(x | y) \log {\frac{p(x | y)}{q(x | y)}} \\
    & = \mathbb{E}_{p(y)}\left[D_{\text{KL}}(p(x | y) \| q(x | y))\right] \\
    & \geq 0
    \end{aligned}
  \end{equation}
\end{proof}
We will formally give a proof of the theorem.
\begin{proof}
We first investigate the optimal intrinsic Bellman operator $\mathcal{T}^{\star}$, and then relate it with any intermediate operator $\mathcal{T}^{\pi_{k}}$.

As defined previously, the optimal policy is $\pi^{\star} = \argmax_{\pi} V^{\pi}$, whose corresponding optimal value function thereby is $V^{\pi^{\star}}$. It should satisfy the intrinsic Bellman equation, therefore $Q^{\pi^{\star}}$ is defined as follows
\begin{equation}
  \label{eq:26}
  Q^{\pi^{\star}}(s, a) = r(s, a) + \gamma \mathbb{E}_{s', a' \sim \pi^{\star}, \mathcal{F}'}\left[Q^{\pi^{\star}}(s', a') + \eta (\log{p^{\pi^{\star}}(a' | \mathcal{F}', s')} - \log{\pi(a' | s')}) \right]
\end{equation}
By Proposition \ref{prop:1}, it turns out that $\mathcal{T}^{\star}Q^{\pi^{\star}} = Q^{\pi^{\star}}$.

Now we relate it to $\mathcal{T}^{\pi_{k}}$, considering $\mathcal{T}^{\star}Q^{\pi_{k}}, \forall k \geq 0$, which can be bounded as
\begin{equation}
  \label{eq:27}
  \begin{aligned}
  \mathcal{T}^{\star}Q^{\pi_{k}} & = r(s, a) + \gamma \mathbb{E}_{s', a' \sim \pi^{\star}, \mathcal{F}'}\left[Q^{\pi_{k}}(s', a') + \eta (\log{p^{\pi^{\star}}(a' | \mathcal{F}', s')} - \log{\pi^{\star}(a' | s')}) \right] \\
  & = r(s, a) + \gamma \mathbb{E}_{s', a' \sim \pi^{\star}, \mathcal{F}'}\left[Q^{\pi_{k}}(s', a') + \eta (\log{\frac{p(\mathcal{F}' | s', a') \pi^{\star}(a' | s')}{\sum\limits_{a'} p(\mathcal{F}' | s', a') \pi^{\star}(a' | s') }} - \log{\pi^{\star}(a' | s')}) \right] \\
  & = r(s, a) + \gamma \mathbb{E}_{s', a' \sim \pi^{\star}, \mathcal{F}'}\left[Q^{\pi_{k}}(s', a') + \eta (\log{\frac{p(\mathcal{F}' | s', a')}{\sum\limits_{a'} p(\mathcal{F}' | s', a') \pi^{\star}(a' | s') }}) \right] \\
  & = r(s, a) + \gamma \mathbb{E}_{s', a' \sim \pi^{\star}, \mathcal{F}'}\left[Q^{\pi_{k}}(s', a') + \eta (\log{\frac{p(\mathcal{F}' | s', a')}{p^{\pi^{\star}}(\mathcal{F}' | s')}}) \right] \\
  & = r(s, a) + \gamma \mathbb{E}_{s', a' \sim \pi^{\star}, \mathcal{F}'}\left[Q^{\pi_{k}}(s', a') + \eta (\log{\frac{p(\mathcal{F}' | s', a') \pi^{\star}(a' | s')}{p^{\pi^{\star}}(\mathcal{F}' | s') \pi^{\star}(a' | s')}}) \right] \\
  & \leq r(s, a) + \gamma \mathbb{E}_{s', a' \sim \pi^{\star}, \mathcal{F}'}\left[Q^{\pi_{k}}(s', a') + \eta (\log{\frac{p(\mathcal{F}' | s', a') \pi^{\star}(a' | s')}{\underbrace{\sum\limits_{a'} p(\mathcal{F}' | s', a') \pi_{k}(a' | s')}_{q(\mathcal{F}' | s')} \pi^{\star}(a' | s')}}) \right] && \vartriangleright \text{by Lemma \ref{lm:kl}} \\
  & = r(s, a) + \gamma \mathbb{E}_{s', a' \sim \pi^{\star}, \mathcal{F}'}\left[Q^{\pi_{k}}(s', a') + \eta (\log{\frac{p(\mathcal{F}' | s', a')}{\sum\limits_{a'} p(\mathcal{F}' | s', a') \pi_{k}(a' | s')} }) \right] \\
  & = r(s, a) + \gamma \mathbb{E}_{s', a' \sim \pi^{\star}, \mathcal{F}'}\left[Q^{\pi_{k}}(s', a') + \eta (\log{\frac{p(\mathcal{F}' | s', a') \pi_{k}(a' | s')}{\sum\limits_{a'} p(\mathcal{F}' | s', a') \pi_{k}(a' | s')} } - \log{\pi_{k}(a' | s')})\right] \\
  & = r(s, a) + \gamma \mathbb{E}_{s', a' \sim \pi^{\star}, \mathcal{F}'}\left[Q^{\pi_{k}}(s', a') + \eta (\log{p^{\pi_{k}}(a' | \mathcal{F}', s')} - \log{\pi_{k}(a' | s')})\right]
  \end{aligned}
\end{equation}
where $q(\mathcal{F}' | s')$ is a well-defined probability since
\begin{equation}
  \label{eq:28}
  \begin{aligned}
      \sum\limits_{\mathcal{F}'}q(\mathcal{F}' | s') & = \sum\limits_{\mathcal{F}'} \sum\limits_{a'} p(\mathcal{F}' | s', a') \pi_{k}(a' | s') \\
  & = \sum\limits_{a'} \sum\limits_{\mathcal{F}'} p(\mathcal{F}' | s', a') \pi_{k}(a' | s') \\
  & = \sum\limits_{a'} 1 \cdot \pi_{k}(a' | s') \\
  & = 1
  \end{aligned}
\end{equation}
By plugging Equation \eqref{eq:23} into Equation \eqref{eq:17} in proof \ref{sec:proof-equat-eqref-5}, we can get
\begin{equation}
  \label{eqq:29}
  \begin{aligned}
  \mathcal{T}^{\pi_{k}}Q^{\pi_{k}}(s, a) & = r(s, a) + \gamma \mathbb{E}_{s'}[V^{\pi_{k}}(s')] \\
  & \leq r(s, a) + \gamma \mathbb{E}_{s'}[V^{\pi_{k + 1}, p^{\pi_{k}}}(s')] \\
  & = r(s, a) + \gamma \mathbb{E}_{s'}\left[\eta\log{\sum\limits_{a'}\exp{\frac{1}{\eta} \left(Q^{\pi_{k}}(s', a') + \eta \mathbb{E}_{\mathcal{F}'}\left[\log{p^{\pi_{k}}(a' | s', \mathcal{F}')}\right]\right)}}\right] \\
  & = r(s, a) + \gamma \mathbb{E}_{s'}\left[\eta\log{Z^{\pi_{k}}(s')}\right] \\
  & \triangleq \mathcal{T}^{\pi_{k + 1}, p^{\pi_{k}}}Q^{\pi_{k}}(s, a)
  \end{aligned}
\end{equation}
With a useful identity from taking logarithm of both sides of $\pi_{k + 1}(a' | s')$ (Equation \eqref{eq:23})
\begin{equation}
  \label{eq:29}
  Q^{\pi_{k}}(s', a') + \eta \log{p^{\pi_{k}}(a' | \mathcal{F}', s')} = \eta(\log{\pi_{k + 1}(a' | s')} + \log{Z^{\pi_{k}}(s')})
\end{equation}
we have an upper bound between $\mathcal{T}^{\star}Q^{\pi_{k}}$ and $\mathcal{T}^{\pi_{k + 1}, p^{\pi_{k}}}Q^{\pi_{k}}$ for $\forall k$
\begin{equation}
  \label{eq:30}
  \begin{aligned}
  & \mathcal{T}^{\star}Q^{\pi_{k}}(s, a) - \mathcal{T}^{\pi_{k + 1}, p^{\pi_{k}}}Q^{\pi_{k}}(s, a)\\
  & \leq r(s, a) + \gamma \mathbb{E}_{s', a' \sim \pi^{\star}, \mathcal{F}'}\left[Q^{\pi_{k}}(s', a') + \eta (\log{p^{\pi_{k}}(a' | \mathcal{F}', s')} - \log{\pi_{k}(a' | s')})\right] - (r(s, a) + \gamma \mathbb{E}_{s'}\left[\eta\log{Z^{\pi_{k}}(s')}\right]) \\
  & = \gamma \mathbb{E}_{s', a' \sim \pi^{\star}, \mathcal{F}'}\left[\eta (\log{\pi_{k + 1}(a' | s')} + \log{Z^{\pi_{k}}(s')}) - \eta \log{\pi_{k}(a' | s')}\right] - \gamma \mathbb{E}_{s'}\left[\eta\log{Z^{\pi_{k}}(s')}\right] \\  
  & = \gamma \mathbb{E}_{s', a' \sim \pi^{\star}}\left[\eta (\log{\pi_{k + 1}(a' | s')} - \log{\pi_{k}(a' | s')})\right]
  \end{aligned}
\end{equation}
Therefore, for an integer $n \geq 1$
\begin{equation}
  \label{eq:31}
  \begin{aligned}
  \frac{1}{n}\sum\limits_{k=0}^{n-1} \mathcal{T}^{\star}Q^{\pi_{k}}(s, a) - \mathcal{T}^{\pi_{k + 1}, p^{\pi_{k}}}Q^{\pi_{k}}(s, a) & \leq \frac{\eta\gamma}{n}\sum\limits_{k=0}^{n-1} \mathbb{E}_{s', a' \sim \pi^{\star}}\left[ \log{\pi_{k + 1}(a' | s')} - \log{\pi_{k}(a' | s')}\right] \\
  & = \frac{\eta\gamma}{n} \mathbb{E}_{s', a' \sim \pi^{\star}}\left[ \log{\frac{\pi_{n}(a' | s')}{\pi_{0}(a' | s')}} \right] \\
  & \leq \frac{\eta\gamma}{n} \mathbb{E}_{s', a' \sim \pi^{\star}}\left[ \log{\frac{\pi^{\star}(a' | s')}{\pi_{0}(a' | s')}} \right]
  \end{aligned}
\end{equation}
where the last inequality is from that cross entropy is always greater than the entropy i.e. $H(\pi^{\star}, \pi_{n}) \geq H(\pi^{\star}), \forall n$, due to non-negativity of KL divergence, and reverse the sign, it follows. By assumption of $H(\pi^{\star})$ being bounded and $\pi_{0}$ non-zero everywhere, the upper bound approaches to zero as $n \rightarrow \infty$.

In the next step, we will find a lower bound on the Equation \eqref{eq:31}.
By Corollary \ref{cor:post}, we have
\begin{equation}
  \label{eq:32}
  \begin{aligned}
  \mathcal{T}^{\pi_{k + 1}, p^{\pi_{k}}}Q^{\pi_{k}}(s, a) & = r(s, a) + \gamma \mathbb{E}_{s', a' \sim \pi_{k + 1}, \mathcal{F}'}\left[Q^{\pi_{k}}(s', a') + \eta (\log{p^{\pi_{k}}(a' | \mathcal{F}', s')} - \log{\pi_{k + 1}(a' | s')})\right] \\
  & \leq r(s, a) + \gamma \mathbb{E}_{s', a' \sim \pi_{k + 1}, \mathcal{F}'}\left[Q^{\pi_{k}}(s', a') + \eta (\log{p^{\pi_{k + 1}}(a' | \mathcal{F}', s')} - \log{\pi_{k + 1}(a' | s')})\right] \\
  & = \mathcal{T}^{\pi_{k + 1}} Q^{\pi_{k}}(s, a)
  \end{aligned}
\end{equation}
Re-implementing the same justifications of Equation \eqref{eq:27}, we further have
\begin{equation}
  \label{eq:33}
  \begin{aligned}
  \mathcal{T}^{\pi_{k + 1}, p^{\pi_{k}}}Q^{\pi_{k}}(s, a) & \leq \mathcal{T}^{\pi_{k + 1}} Q^{\pi_{k}}(s, a) \\
  & \leq r(s, a) + \gamma \mathbb{E}_{s', a' \sim \pi_{k + 1}, \mathcal{F}'}\left[Q^{\pi_{k}}(s', a') + \eta (\log{p^{\pi^{\star}}(a' | \mathcal{F}', s')} - \log{\pi^{\star}(a' | s')})\right]
  \end{aligned}
\end{equation}
Therefore
\begin{equation}
  \label{eq:34}
  \begin{aligned}
  & \mathcal{T}^{\star}Q^{\pi_{k}}(s, a) - \mathcal{T}^{\pi_{k + 1}, p^{\pi_{k}}}Q^{\pi_{k}}(s, a)\\
  & \geq r(s, a) + \gamma \mathbb{E}_{s', a' \sim \pi^{\star}, \mathcal{F}'}\left[Q^{\pi_{k}}(s', a') + \eta (\log{p^{\pi^{\star}}(a' | \mathcal{F}', s')} \log{\pi^{\star}(a' | s')})\right] - \\
  &\quad\  (r(s, a) + \gamma \mathbb{E}_{s', a' \sim \pi_{k + 1}, \mathcal{F}'}\left[Q^{\pi_{k}}(s', a') + \eta (\log{p^{\pi^{\star}}(a' | \mathcal{F}', s')} - \log{\pi^{\star}(a' | s')})\right]) \\
  & = \gamma \mathbb{E}_{s', (a', \mathcal{F}') \sim ((p^{\pi^{\star}}(\mathcal{F}', a' | s') - p^{\pi_{k + 1}}(\mathcal{F}', a' | s')))}\left[\eta (\log{p^{\pi^{\star}}(a' | \mathcal{F}', s')} - \log{\pi^{\star}(a' | s')}) + Q^{\pi_{k}}\right]\footnotemark
  \end{aligned}
\end{equation}
\footnotetext{We merge $(a', \mathcal{F}')$ together for the reason that $p(\mathcal{F}' | s', a')$ may have a complex dependency on $\pi$ as $\mathcal{F}$ becomes longer.}
Summing together, we have
\begin{equation}
  \label{eq:35}
  \begin{aligned}
  & \frac{1}{n}\sum\limits_{k=0}^{n-1} \mathcal{T}^{\star}Q^{\pi_{k}}(s, a) - \mathcal{T}^{\pi_{k + 1}, p^{\pi_{k}}}Q^{\pi_{k}}(s, a) \\
  & \geq \gamma \frac{1}{n}\sum\limits_{k=0}^{n-1} \mathbb{E}_{s', (a', \mathcal{F}') \sim (p^{\pi^{\star}}(\mathcal{F}', a' | s') - p^{\pi_{k + 1}}(\mathcal{F}', a' | s'))}\left[\eta (\log{p^{\pi^{\star}}(a' | \mathcal{F}', s')} - \log{\pi^{\star}(a' | s')}) + Q^{\pi_{k}}\right]
  \end{aligned}
\end{equation}
Since $\lim_{n \rightarrow \infty}\sum\limits_{k=0}^{n}\mathbb{E}_{s', (a', \mathcal{F}') \sim ((p^{\pi^{\star}}(\mathcal{F}', a' | s') - p^{\pi_{k + 1}}(\mathcal{F}', a' | s')))}\left[\eta (\log{p^{\pi^{\star}}(a' | \mathcal{F}', s')} - \log{\pi^{\star}(a' | s')}) + Q^{\pi_{k}}\right]$ exists, the lower bound approaches to zero as $n \rightarrow \infty$.

Combining those two ends, we conclude that $\frac{1}{n}\sum\limits_{k=0}^{n-1} \mathcal{T}^{\star}Q^{\pi_{k}}(s, a) - \mathcal{T}^{\pi_{k + 1}, p^{\pi_{k}}}Q^{\pi_{k}}(s, a)$ approaches zero as $n \rightarrow \infty$, which implies $\lim_{n \rightarrow \infty} \sum\limits_{k=0}^{n-1} \mathcal{T}^{\star}Q^{\pi_{k}}(s, a) - \mathcal{T}^{\pi_{k + 1}, p^{\pi_{k}}}Q^{\pi_{k}}(s, a)$ exists. It immediately follows that $\lim_{n \rightarrow \infty} \left(\mathcal{T}^{\star}Q^{\pi_{k}}(s, a) - \mathcal{T}^{\pi_{k + 1}, p^{\pi_{k}}}Q^{\pi_{k}}(s, a)\right) = 0$. It is also held for $\lim_{n \rightarrow \infty} \left(\mathcal{T}^{\star}Q^{\pi_{k}}(s, a) - \mathcal{T}^{\pi_{k + 1}}Q^{\pi_{k}}(s, a)\right) = 0$, since $\mathcal{T}^{\pi_{k + 1}}$ is bounded below by $\mathcal{T}^{\pi_{k + 1}, p^{\pi_{k}}}$.
We also note
\begin{equation}
  \label{eq:36}
  \lim_{n \rightarrow \infty} \mathcal{T}^{\pi_{k + 1}}Q^{\pi_{k}}(s, a) = \mathcal{T}^{\pi_{\infty}}Q^{\pi_{\infty}}(s, a) = Q^{\pi_{\infty}}(s, a)
\end{equation}
And it follows that $\|\mathcal{T}^{\star}(Q^{\pi_{k}} - Q^{\pi_{\infty}})\| \leq \|\mathcal{T}^{\star}\| \|Q^{\pi_{k}} - Q^{\pi_{\infty}} \|$. Since $\mathcal{T}^{\pi^{\star}}$ is a bounded linear operator, and $Q^{\pi_{k}} \rightarrow Q^{\pi_{\infty}}$, it implies that $\lim_{n \rightarrow \infty} \mathcal{T}^{\star}Q^{\pi_{k}}(s, a) = \mathcal{T}^{\star}Q^{\pi_{\infty}}(s, a)$. Comparing those terms, we have $\mathcal{T}^{\star}Q^{\pi_{\infty}}(s, a) = Q^{\pi_{\infty}}(s, a)$. However, since $\mathcal{T}^{\pi^{\star}}$ has a unique fixed point, it implies that $Q^{\pi_{\infty}}(s, a) = Q^{\pi^{\star}}(s, a)$.  
\end{proof}
\section{Proof of Theorem \ref{thm:monotonic}}
\begin{proof}
Since $\tilde{\pi}$ minimizes the projection loss, then it follows that
\begin{equation}
  \label{eq:37}
  \begin{aligned}
  & \mathbb{E}_{a_{t} \sim \tilde{\pi}, \mathcal{F}_{t}}\left[\eta (\log{\tilde{\pi}(a_{t} | s_{t})} - \log p^{\pi}(a_{t} | \mathcal{F}_{t}, s_{t})) - Q^{\pi}(s_{t}, a_{t}) + \eta \log{Z^{\pi}(s_{t})}\right] \\
  & \leq \mathbb{E}_{a_{t} \sim \pi, \mathcal{F}_{t}}\left[\eta (\log{\pi(a_{t} | s_{t})} - \log p^{\pi}(a_{t} | \mathcal{F}_{t}, s_{t})) - Q^{\pi}(s_{t}, a_{t}) + \eta \log{Z^{\pi}(s_{t})}\right]    
  \end{aligned}
\end{equation}
Since the partition function is dependent only on state and not relies on $\tilde{\pi}$, thus it can be canceled out from both sides. Rearranging, we have
\begin{equation}
  \label{eq:38}
  V^{\pi}(s_{t}) \leq \mathbb{E}_{a_{t} \sim \tilde{\pi}, \mathcal{F}_{t}}\left[Q^{\pi}(s_{t}, a_{t}) + \eta (\log{p^{\pi}(a_{t} | \mathcal{F}_{t}, s_{t})} - \log{\tilde{\pi}(a_{t} | s_{t})})\right]
\end{equation}
Define $\mathcal{I}^{\tilde{\pi}, p^{\pi}}(a, \mathcal{F} | s)$ as follows
\begin{equation}
  \label{eq:39}
\begin{aligned}
  \mathcal{I}^{\tilde{\pi}, p^{\pi}}(a, \mathcal{F}| s) = \mathbb{E}_{\tilde{\pi}(a | s) p(\mathcal{F} | s, a)}\left[\log \frac{p^{\pi}(a | \mathcal{F}, s)}{\tilde{\pi}(a | s)}\right]
\end{aligned}
\end{equation}
By repeatedly applying above inequality, we have
\begin{equation}
  \label{eq:40}
  \begin{aligned}
      Q^{\pi}(s_{t}, a_{t}) & = r(s_{t}, a_{t}) + \gamma \mathbb{E}_{s_{t + 1}}[V^{\pi}(s_{t + 1})] \\
  & \leq r(s_{t}, a_{t}) + \gamma \mathbb{E}_{s_{t + 1}}[\mathbb{E}_{a_{t + 1} \sim \tilde{\pi}, \mathcal{F}_{t + 1}}\left[Q^{\pi}(s_{t + 1}, a_{t + 1}) + \eta (\log{p^{\pi}(a_{t + 1} | \mathcal{F}_{t + 1}, s_{t + 1})} - \log{\tilde{\pi}(a_{t + 1} | s_{t + 1})})\right] \\
  & = r(s_{t}, a_{t}) + \gamma \eta \mathbb{E}_{s_{t + 1}} \left[\mathcal{I}^{\tilde{\pi}, p^{\pi}}(a_{t + 1}, \mathcal{F}_{t + 1}| s_{t + 1})\right] + \gamma \mathbb{E}_{s_{t + 1}}\left[\mathbb{E}_{a_{t + 1} \sim \tilde{\pi}, , \mathcal{F}_{t + 1}}\left[r(s_{t + 1}, a_{t + 1}) + \gamma \mathbb{E}_{s_{t + 2}}[V^{\pi}(s_{t + 2})]\right]\right] \\
  & \vdots \\
  & \leq \mathbb{E}_{s_{t + 1}, a_{t + 1}, \dots | \tilde{\pi}}\Biggl[\sum\limits_{l=0}^{\infty}\gamma^{l}(r(s_{t + l}, a_{t + l}) + \gamma \alpha \eta \mathbb{E}_{s_{t + l + 1}} \left[\mathcal{I}^{\tilde{\pi}, p^{\pi}}(a_{t + l + 1}, \mathcal{F}_{t + l + 1}| s_{t + l + 1})\right])\Biggr] \\
  & \triangleq Q^{\tilde{\pi}, p^{\pi}}(s_{t}, a_{t})
  \end{aligned}
\end{equation}
By Corollary \ref{cor:post}, we have
\begin{equation}
  \label{eq:41}
  \mathcal{I}^{\tilde{\pi}, p^{\pi}}(a, \mathcal{F}| s) \leq \mathcal{I}^{\tilde{\pi}}(a, \mathcal{F}| s)
\end{equation}
Therefore
\begin{equation}
  \label{eq:42}
  Q^{\pi}(s_{t}, a_{t}) \leq Q^{\tilde{\pi}, p^{\pi}}(s_{t}, a_{t}) \leq Q^{\tilde{\pi}}(s_{t}, a_{t})
\end{equation}  
\end{proof}
\section{Derivation of Lower Bounds}
\subsection{One-step Lower Bound}
We will present a more general lower bound considering future sequence $\mathcal{F}$ by using importance sampling and Jensen's inequality
\begin{equation}
  \label{eq:43}
  \begin{aligned}
  \log p^{\pi}(\mathcal{F} | s) & = \log \int_{a} p^{\pi}(\mathcal{F}, a | s) da \\
  & = \log \mathbb{E}_{a \sim q_{\phi}(a | \mathcal{F}, s)}\left[ \frac{p^{\pi}(\mathcal{F}, a | s)}{q_{\phi}(a | \mathcal{F}, s)}\right]  \\
  & \geq \mathbb{E}_{a \sim q_{\phi}(a | \mathcal{F}, s)} \left[\log{\frac{p^{\pi}(\mathcal{F}, a | s)}{q_{\phi}(a | \mathcal{F}, s)}} \right] \\
  & = \mathbb{E}_{q_{\phi}(a | \mathcal{F}, s)} \left[\log{p_{\psi}(\mathcal{F} | s, a)}\right] - D_{\text{KL}}(q_{\phi}(a | \mathcal{F}, s) || \pi(a | s))    
  \end{aligned}
\end{equation}
When $\mathcal{F} = (s', r)$, we can get the one-step variational lower bound
\begin{equation}
  \label{eq:44}
  \begin{aligned}
    \log p^{\pi}(s', r | s) & \geq \mathcal{L}(\phi, \psi; s, s', r)\\
  & = -D_{\text{KL}}(q_{\phi}(a | s, s', r) || \pi(a | s)) + \mathbb{E}_{q_{\phi}(a | s, s', r)}[\log p_{\psi}(s', r | s, a)]    
  \end{aligned}
\end{equation}
\subsection{Trajectory-Wise Lower Bound}
Denote $\mathbf{x}_{1:T} = {(s_{t + 1}, r_{t} | s_{t})}_{t = 1}^{T}$ and $\mathbf{z}_{1:T} = {(a_{t} | s_{t})}_{t = 1}^{T}$, we assume the joint distribution $p^{\pi} \left(\mathbf{x}_{1:T}, \mathbf{z}_{1:T} \right)$ and $q_{\phi} \left( \mathbf{z}_{1:T} \vert \mathbf{x}_{1:T} \right)$ can be factorized as follows
\begin{equation}
  \label{eq:45}
  \begin{aligned}
  p^{\pi} \left(\mathbf{x}_{1:T}, \mathbf{z}_{1:T} \right) & = \prod_{\tau = 1}^{T} p_{\psi}(s'_{\tau}, r_{\tau} | s_{\tau}, a_{\tau}) p^{\pi}(a_{\tau} | s_{1:\tau}, a_{1:\tau - 1}) \\
  & = \prod_{\tau = 1}^{T} p_{\psi}(s'_{\tau}, r_{\tau} | s_{\tau}, a_{\tau}) \pi(a_{\tau} | s_{\tau}) && \vartriangleright \text{by Markov property}
  \end{aligned}
\end{equation}

\begin{equation}
  \label{eq:46}
  q_{\phi} \left( \mathbf{z}_{1:T} \vert \mathbf{x}_{1:T} \right) = \prod_{\tau = 1}^{T} q_{\phi}(a_{\tau} | s_{1:\tau + 1}, r_{1:\tau}, a_{1:\tau - 1})
\end{equation}
In a similar fashion
\begin{equation}
  \label{eq:47}
  \begin{aligned}
  \log p^{\pi} \left(\mathbf{x}_{1:T} \right) & = \int_{\mathbf{z}_{1:T}} p \left(\mathbf{x}_{1:T}, \mathbf{z}_{1:T} \right) d\mathbf{z}_{1:T} \\
  &= \log \mathbb{E}_{\mathbf{z}_{1:T} \sim q_\phi \left( \mathbf{z}_{1:T} \vert \mathbf{x}_{1:T}\right)} \left[ \frac{p \left(\mathbf{x}_{1:T}, \mathbf{z}_{1:T} \right) }{q_\phi \left( \mathbf{z}_{1:T} \vert \mathbf{x}_{1:T}\right)} \right] \\
  &\geq \mathbb{E}_{\mathbf{z}_{1:T}} \left[\log \frac{p \left(\mathbf{x}_{1:T}, \mathbf{z}_{1:T} \right) }{q_\phi \left( \mathbf{z}_{1:T} \vert \mathbf{x}_{1:T}\right)}\right] \\
  &= \mathbb{E}_{\mathbf{z}_{1:T}} \left[\log \frac{\prod_{\tau = 1}^{T} p_{\psi}(s'_{\tau}, r_{\tau} | s_{\tau}, a_{\tau}) \pi(a_{\tau} | s_{\tau})}{\prod_{\tau = 1}^{T} q_{\phi}(a_{\tau} | s_{1:\tau + 1}, r_{1:\tau}, a_{1:\tau - 1})}\right] \\
  &= \mathbb{E}_{\mathbf{z}_{1:T}} \left[\sum_{\tau=1}^T \log p_{\psi}(s'_{\tau}, r_{\tau} | s_{\tau}, a_{\tau}) + \log \pi(a_{\tau} | s_{\tau}) - \log q_{\phi}(a_{\tau} | s_{1:\tau + 1}, r_{1:\tau}, a_{1:\tau - 1})  \right] \\
  &= \sum_{\tau=1}^{T} \mathbb{E}_{\mathbf{z}_{1:\tau}} \left[\log p_{\psi} \left(s'_{\tau}, r_{\tau} | s_{\tau}, a_{\tau}\right) - D_{KL}\left(q_\phi \left(a_{\tau} | s_{1:\tau + 1}, r_{1:\tau}, a_{1:\tau - 1}\right)\Vert \pi \left(a_{\tau} | s_{\tau}\right)  \right) \right]
  \end{aligned}
\end{equation}

If we further assume conditional independence of the past information i.e. $t < \tau$ for $q_{\phi}$, then we have $q_\phi \left(a_{\tau} | s_{1:\tau + 1}, r_{1:\tau}, a_{1:\tau - 1}\right) = q_{\phi} \left(a_{\tau} | s_{\tau + 1}, r_{\tau}, s_{\tau}\right)$. The above formulation then deduces to
\begin{equation}
  \label{eq:48}
  \sum_{\tau=1}^{T} \mathbb{E}_{\mathbf{z}_{1:\tau}} \left[\log p_{\psi} \left(s'_{\tau}, r_{\tau} | s_{\tau}, a_{\tau}\right) - D_{KL}\left(q_\phi \left(a_{\tau} | s_{\tau + 1}, r_{\tau}, s_{\tau}\right) \Vert \pi \left(a_{\tau} | s_{\tau}\right)  \right) \right]
\end{equation}
which is simply a summation of a series of one-step lower bounds as derived earlier. This is helpful since we can employ the same model architecture while explore different training procedures, such as being more on-policy to capture trajectory's information.

\end{document}